\documentclass[]{article}
\usepackage{proceed}
\usepackage{amsmath}
\usepackage{graphicx}
\usepackage{amssymb}
\usepackage{epsfig}
\usepackage{url}
\usepackage{verbatim}

\usepackage{times}
\usepackage{helvet}
\usepackage{courier}
\usepackage{epsfig}

\usepackage{algorithmic}
\usepackage{algorithm}
\usepackage{natbib}
\usepackage{amsfonts}
\usepackage{amsmath,amsthm}

\newcounter{thm_counter}
\newcounter{lem_counter}
\newcounter{pro_counter}
\newcounter{ass_counter}

\newtheorem{theorem}[thm_counter]{Theorem}
\newtheorem{proposition}[pro_counter]{Proposition}
\newtheorem{lemma}[lem_counter]{Lemma}
\newtheorem{assumption}[ass_counter]{Assumption}
\newtheorem{definition}[ass_counter]{Definition}

\title{Finite-Sample Analysis of Proximal Gradient TD Algorithms}

{\footnotesize
\author{
Bo Liu \\
UMass Amherst\\
{\scriptsize \texttt{boliu@cs.umass.edu}}
\And
Ji Liu\\
University of Rochester\\
{\scriptsize \texttt{jliu@cs.rochester.edu}}
\And
Mohammad Ghavamzadeh \\
Adobe \& INRIA Lille\\
\quad {\scriptsize \texttt{Mohammad.ghavamzadeh@inria.fr\quad}} 
\And
Sridhar Mahadevan\\
UMass Amherst\\
\quad {\scriptsize \texttt{mahadeva@cs.umass.edu}}
\And
Marek Petrik\\
IBM Research\\
{\scriptsize \texttt{marekpetrik@gmail.com}}
}
}

\begin{document}
\maketitle

\begin{abstract} 
 In this paper, we show for the first time how gradient TD  (GTD) reinforcement learning methods can be formally derived as true stochastic gradient algorithms, not with respect to their original objective functions as previously attempted, but rather using derived primal-dual saddle-point objective functions. We  then conduct a saddle-point error analysis to obtain finite-sample bounds on their performance.  Previous analyses of this class of algorithms use stochastic approximation techniques to prove asymptotic convergence,  and no finite-sample analysis had been  attempted.  Two novel GTD algorithms are also proposed, namely projected GTD2 and GTD2-MP, which use proximal ``mirror maps'' to yield  improved convergence guarantees and acceleration, respectively.  The results of our theoretical analysis imply that the GTD family of algorithms are  comparable and may indeed be preferred over existing least squares TD methods for off-policy learning, due to their linear complexity. We provide experimental results showing the improved performance of our accelerated gradient TD methods. 
\end{abstract} 


\section{INTRODUCTION}

Obtaining a true stochastic gradient temporal difference method has been a longstanding goal of reinforcement learning (RL) ~\citep{ndp:book,sutton-barto:book}, ever since it was discovered that the original TD method was unstable in many off-policy scenarios where the target behavior being learned and the exploratory behavior producing samples differ. 
\citet{Sutton:GTD1:2008,FastGradient:2009} proposed the family of gradient-based temporal difference (GTD) algorithms which offer several  interesting properties. A key property of this class of GTD algorithms is that they are asymptotically off-policy convergent, which  was shown using stochastic approximation ~\citep{borkar:book}. This is quite important when we notice that many RL algorithms, especially those that are based on stochastic approximation, such as TD($\lambda$), do not have convergence guarantees in the off-policy setting.
Unfortunately, this class of GTD algorithms are {\em not true stochastic gradient methods with respect to their original objective functions}, as pointed out in~\citet{szepesvari2010algorithms}. The reason is not surprising: the gradient of the objective functions used involve products of terms, which cannot be sampled directly, and was decomposed by a rather ad-hoc splitting of terms. In this paper, we take a major step forward in resolving this problem by showing a principled way of designing true stochastic gradient TD algorithms by using a primal-dual saddle point objective function, derived from the original objective functions, coupled with the principled use of {\em operator splitting}~\citep{BOOK2011PROXSPLIT}. 


Since in real-world applications of RL, we  have access to only a finite amount of data, finite-sample analysis of gradient TD algorithms is essential as it clearly shows the effect of the number of samples (and the parameters that play a role in the sampling budget of the algorithm) in their final performance. However, most of the work on finite-sample analysis in RL has been focused on  batch RL (or approximate dynamic programming) algorithms (e.g.,~\citealt{Kakade02AO,Munos08FT,Antos06learningnear-optimal,Lazaric10AC}), especially those that are  least squares TD (LSTD)-based (e.g.,~\citealt{Lazaric_finite-sampleanalysis,lstdrp:nips2010,LASSOTD:2011,Lazaric12FS}), and more importantly restricted to  the on-policy setting. In this paper, we provide the finite-sample analysis of the GTD family of  algorithms, a relatively novel class of gradient-based TD methods that are guaranteed to converge even in the off-policy setting, and for which, to the best of our knowledge, no finite-sample analysis has been reported. This analysis is challenging because {\bf 1)} the stochastic approximation methods that have been used to prove the asymptotic convergence of these algorithms do not address convergence rate analysis;  {\bf 2)} as we explain  in detail in Section~\ref{subsec:GTD-algos},  the techniques used for the analysis of the stochastic gradient methods cannot be applied here;  {\bf 3)} finally, the difficulty of finite-sample analysis in the off-policy setting. 

The major contributions of this paper include the first  finite-sample analysis of the class of gradient TD algorithms, as well as the design and analysis of several improved GTD methods that result from our novel approach of  formulating  gradient TD methods as true stochastic gradient algorithms w.r.t.~a saddle-point objective function.  We then use the techniques applied in the analysis of the stochastic gradient methods to propose a unified finite-sample analysis for the previously proposed as well as our novel gradient TD algorithms. Finally, given the results of our analysis, we study the GTD class of algorithms from several different perspectives, including acceleration in convergence, learning with biased importance sampling factors, etc.


\section{PRELIMINARIES}
\label{sec:preliminaries}

Reinforcement Learning (RL)~\citep{ndp:book,sutton-barto:book} is a class of learning problems in which an agent interacts with an unfamiliar, dynamic and stochastic environment, where the agent's goal is to optimize some measure of its long-term performance. This interaction is conventionally modeled as a Markov decision process (MDP). A MDP is defined as the tuple $({\mathcal{S},\mathcal{A},P_{ss'}^{a},R,\gamma})$, where $\mathcal{S}$ and $\mathcal{A}$ are the sets of states and actions, the transition kernel $P_{ss'}^{a}$ specifying the probability of transition from state $s\in\mathcal{S}$ to state $s'\in\mathcal{S}$ by taking action $a\in\mathcal{A}$, $R(s,a):\mathcal{S}\times\mathcal{A}\to\mathbb{R}$ is the reward function bounded by $R_{\max}$., and $0\leq\gamma<1$ is a discount factor. A stationary policy $\pi:\mathcal{S}\times\mathcal{A}\to\left[{0,1}\right]$ is a probabilistic mapping from states to actions. The main objective of a RL algorithm is to find an optimal policy. In order to achieve this goal, a key step in many algorithms is to calculate the value function of a given policy $\pi$, i.e.,~$V^{\pi}:\mathcal{S}\to\mathbb{R}$, a process known as {\em policy evaluation}. It is known that $V^\pi$ is the unique fixed-point of the {\em Bellman operator} $T^\pi$, i.e.,
\begin{equation}
\label{eq:BellmanEq}
V^\pi = T^\pi V^\pi = R^\pi + \gamma P^\pi V^\pi,
\end{equation}
where $R^\pi$ and $P^\pi$ are the reward function and transition kernel of the Markov chain induced by policy $\pi$. In Eq.~\ref{eq:BellmanEq}, we may imagine $V^\pi$ as a $|\mathcal{S}|$-dimensional vector and write everything in vector/matrix form. In the following, to simplify the notation, we often drop the dependence of $T^\pi$, $V^\pi$, $R^\pi$, and $P^\pi$ to $\pi$. 

We denote by $\pi_b$, the behavior policy that generates the data, and by $\pi$, the target policy that we would like to evaluate. They are the same in the on-policy setting and different in the off-policy scenario. For each state-action pair $(s_i,a_i)$, such that $\pi_b(a_i|s_i)>0$, we define the importance-weighting factor $\rho_i = \pi(a_i|s_i)/\pi _b(a_i|s_i)$ with $\rho_{\max}\geq 0$ being its maximum value over the state-action pairs.

When $\mathcal{S}$ is large or infinite, we often use a linear approximation architecture for $V^\pi$ with parameters $\theta\in\mathbb{R}^d$ and $L$-bounded basis functions $\{\varphi_i\}_{i=1}^d$, i.e.,~$\varphi_i:\mathcal{S}\rightarrow\mathbb{R}$ and $\max_i||\varphi_i||_\infty\leq L$. We denote by $\phi(\cdot)=\big(\varphi_1(\cdot),\ldots,\varphi_d(\cdot)\big)^\top$ the feature vector and by $\mathcal{F}$ the linear function space spanned by the basis functions $\{\varphi_i\}_{i=1}^d$, i.e.,~$\mathcal{F}=\big\{f_\theta\mid\theta\in\mathbb{R}^d\;\text{and}\;f_\theta(\cdot)=\phi(\cdot)^\top\theta\big\}$. We may write the approximation of $V$ in $\mathcal{F}$ in the vector form as $\hat{v}=\Phi\theta$, where $\Phi$ is the $|\mathcal{S}|\times d$ feature matrix. When only $n$ training samples of the form $\mathcal{D}=\big\{\big(s_i,a_i,r_i=r(s_i,a_i),s'_i\big)\big\}_{i=1}^n,\;s_i\sim\xi,\;a_i\sim\pi_b(\cdot|s_i),\;s'_i\sim P(\cdot|s_i,a_i)$, are available ($\xi$ is a distribution over the state space $\mathcal{S}$), we may write the {\em empirical Bellman operator} $\hat{T}$ for a function in $\mathcal{F}$ as 
\begin{equation}
\label{eq:EmpBellmanEq}
\hat{T}(\hat \Phi \theta ) = \hat R + \gamma\hat\Phi '\theta,
\end{equation}
where $\hat{\Phi}$ (resp.~$\hat{\Phi}'$) is the empirical feature matrix of size $n\times d$, whose $i$-th row is the feature vector $\phi(s_i)^\top$ (resp.~$\phi(s'_i)^\top$), and $\hat{R}\in\mathbb{R}^n$ is the reward vector, whose $i$-th element is $r_i$. We denote by $\delta_i(\theta)=r_i+\gamma\phi_i^{'\top}\theta-\phi_i^\top\theta$, the TD error for the $i$-th sample $(s_i,r_i,s'_i)$ and define $\Delta\phi_i=\phi_i-\gamma\phi'_i$. Finally, we define the matrices $A$ and $C$, and the vector $b$ as

\vspace{-0.15in}
\begin{small}
\begin{equation}
A := \mathbb{E}\big[\rho_i\phi_i(\Delta\phi_i)^\top\big],\;\;b := \mathbb{E}\left[\rho_i\phi_ir_i\right],\;\;C := \mathbb{E}[\phi_i\phi_i^\top],
\label{eq:abc}   
\end{equation}
\end{small}
\vspace{-0.2in}

where the expectations are w.r.t.~$\xi$ and $P^{\pi_b}$. We also denote by $\Xi$, the diagonal matrix whose elements are $\xi(s)$, and ${\xi _{\max }} := {\max _s}\xi (s)$. For each sample $i$ in the training set $\mathcal{D}$, we can calculate an unbiased estimate of $A$, $b$, and $C$ as follows:
\begin{equation}
\hat{A}_i := \rho_i\phi_i\Delta\phi_i^\top, \quad\; \hat{b}_i := \rho_ir_i\phi_i, \quad\; \hat{C}_i := \phi_i\phi_i^\top.
\label{eq:atbtct}
\end{equation}


\subsection{GRADIENT-BASED TD ALGORITHMS}
\label{subsec:GTD-algos}

The class of gradient-based TD (GTD) algorithms were proposed by~\citet{Sutton:GTD1:2008,FastGradient:2009}. These algorithms target two objective functions: the {\em norm of the expected TD update} (NEU) and the {\em mean-square projected Bellman error} (MSPBE), defined as (see e.g.,~\citealt{maei2011gradient})\footnote{It is important to note that $T$ in~\eqref{eq:neu} and~\eqref{eq:mspbe} is $T^\pi$, the Bellman operator of the target policy $\pi$.}

\vspace{-0.15in}
\begin{small}
\begin{align}
{\rm {NEU}}(\theta) &= ||\Phi^\top\Xi(T\hat{v}-\hat{v})||^{2}\;,
\label{eq:neu}\\
{\rm {MSPBE}}(\theta) 
&=  ||\hat{v} - \Pi T\hat{v}||_{\xi}^2 = ||\Phi^\top\Xi(T\hat{v}-\hat{v})||_{C^{-1}}^2\;,
\label{eq:mspbe}
\end{align}
\end{small}
\vspace{-0.2in}

where $C=\mathbb{E}[\phi_i\phi_i^\top]=\Phi^\top\Xi\Phi$ is the covariance matrix defined in Eq.~\ref{eq:abc} and is assumed to be non-singular, and $\Pi = \Phi(\Phi ^\top\Xi\Phi)^{-1}\Phi^\top\Xi$ is the orthogonal projection operator into the function space $\mathcal{F}$, i.e.,~for any bounded function $g,\;\Pi g=\arg\min_{f\in\mathcal{F}}||g-f||_{\xi}$. From~\eqref{eq:neu} and~\eqref{eq:mspbe}, it is clear that NEU and MSPBE are square unweighted and weighted by $C^{-1}$, $\ell_2$-norms of the quantity $\Phi^\top\Xi(T\hat{v}-\hat{v})$, respectively, and thus, the two objective functions can be unified as
\begin{equation}
J(\theta)=||\Phi^\top\Xi(T\hat{v}-\hat{v})||_{M^{-1}}^{2} = ||\mathbb{E}[\rho_i\delta_i(\theta)\phi_i]||_{M^{-1}}^{2},
\label{eq:j}
\end{equation}
with $M$ equals to the identity matrix $I$ for NEU and to the covariance matrix $C$ for MSPBE. The second equality in~\eqref{eq:j} holds because of the following lemma from Section 4.2 in~\citet{maei2011gradient}.
\begin{lemma}
\label{lem:e}
Let $\mathcal{D}=\big\{\big(s_i,a_i,r_i,s'_i\big)\big\}_{i=1}^n,\;s_i\sim\xi,\;a_i\sim\pi_b(\cdot|s_i),\;s'_i\sim P(\cdot|s_i,a_i)$ be a training set generated by the behavior policy $\pi_b$ and $T$ be the Bellman operator of the target policy $\pi$. Then, we have 
\begin{equation*}
\Phi^\top\Xi (T\hat v - \hat v) = \mathbb{E}\big[\rho_i\delta_i(\theta)\phi_i\big] = b-A\theta.
\end{equation*}
\end{lemma}

Motivated by minimizing the NEU and MSPBE objective functions using the stochastic gradient methods, the GTD and GTD2 algorithms were proposed with the following update rules: 
\begin{align}
\label{eq:gtd}
\hspace{-1.75cm}\textbf{GTD:}\quad\quad y_{t + 1} &= y_t + \alpha_t\big(\rho_t\delta_t(\theta_t)\phi_t - y_t\big), \\
\theta_{t + 1} &= \theta_t + \alpha_t\rho_t\Delta\phi_t(y_t^\top\phi_t), \nonumber
\end{align}
\begin{align}
\label{eq:gtd2}
\hspace{-1cm}\textbf{GTD2:}\quad\quad y_{t + 1} &= y_t + \alpha_t\big(\rho_t\delta_t(\theta_t) - \phi_t^\top y_t\big)\phi_t,\\
\theta_{t + 1} &= \theta_t + \alpha_t\rho_t\Delta\phi_t(y_t^\top\phi_t). \nonumber
\end{align}
However, it has been shown that the above update rules do not update the value function parameter $\theta$ in the gradient direction of NEU and MSPBE, and thus, NEU and MSPBE are not the true objective functions of the GTD and GTD2 algorithms~\citep{szepesvari2010algorithms}. Consider the NEU objective function in~\eqref{eq:neu}. Taking its gradient w.r.t.~$\theta$, we obtain
\begin{eqnarray}
\nonumber
 - \frac{1}{2}\nabla {\rm{NEU}}(\theta ) &=&  - \big(\nabla\mathbb{E}\big[\rho_i\delta_i(\theta)\phi^\top_i\big]\big)\mathbb{E}\big[\rho_i\delta_i(\theta)\phi_i\big] \\
 \nonumber
 &=&  - \big(\mathbb{E}\big[\rho_i\nabla\delta_i(\theta)\phi^\top_i\big]\big)\mathbb{E}\big[\rho_i\delta_i(\theta)\phi_i\big]\\
 &=& \mathbb{E}\big[\rho_i\Delta\phi_i\phi_i^\top\big]\mathbb{E}\big[\rho_i\delta_i(\theta)\phi_i\big].
 \label{eq:neu-grad}
\end{eqnarray}
If the gradient can be written as a single expectation, then it is straightforward to use a stochastic gradient method. However, we have a product of two expectations in~\eqref{eq:neu-grad}, and unfortunately, due to the correlation between them, the sample product (with a single sample) won't be an unbiased estimate of the gradient. To tackle this, the GTD algorithm uses an auxiliary variable $y_t$ to estimate $\mathbb{E}\big[\rho_i\delta_i(\theta)\phi_i\big]$, and thus, the overall algorithm is no longer a true stochastic gradient method w.r.t.~NEU. It can be easily shown that the same problem exists for GTD2 w.r.t.~the MSPBE objective function. This prevents us from using the standard convergence analysis techniques of stochastic gradient descent methods to obtain a finite-sample performance bound for the GTD and GTD2 algorithms.


It should be also noted that in the original publications of GTD/GTD2 algorithms~\citep{Sutton:GTD1:2008,FastGradient:2009}, the authors discussed handling the off-policy scenario using both importance and rejected sampling. In rejected sampling that was mainly used in~\citet{Sutton:GTD1:2008,FastGradient:2009}, a sample $({s_i},{a_i},{r_i},s'_i)$ is rejected and the parameter $\theta$ does not update for this sample, if $\pi(a_i|s_i) = 0$. This sampling strategy is not efficient since a lot of samples will be discarded if $\pi_b$ and $\pi$ are very different. 


\subsection{RELATED WORK}

\vskip -1.3cm
Before we present a finite-sample performance bound for GTD and GTD2, it would be helpful to give a brief overview of the existing literature on finite-sample analysis of the TD algorithms. The convergence rate of the TD algorithms mainly depends on $(d,n,\nu)$, where $d$ is the size of the approximation space (the dimension of the feature vector), $n$ is the number of samples, and $\nu$ is the smallest eigenvalue of the sample-based covariance matrix $\hat C=\hat\Phi^\top\hat\Phi$, i.e.,~$\nu=\lambda_{\min }(\hat C)$.

\citet{Antos06learningnear-optimal} proved an error bound of $O(\frac{d\log d}{n^{1/4}})$ for LSTD in bounded spaces.~\citet{Lazaric_finite-sampleanalysis} proposed a LSTD analysis in learner spaces and obtained a tighter bound of $O(\sqrt{\frac{d\log d}{n\nu })}$ and later used it to derive a bound for the least-squares policy iteration (LSPI) algorithm~\citep{Lazaric12FS}.~\citet{bruno:lstdlambda} recently proposed the first convergence analysis for LSTD$(\lambda)$ and derived a bound of $\tilde O(d/\nu\sqrt n )$. The analysis is a bit different than the one in~\citet{Lazaric_finite-sampleanalysis} and the bound is weaker in terms of $d$ and $\nu$. Another recent result is by~\citet{drift:prashanth2014fast} that use stochastic approximation to solve LSTD$(0)$, where the resulting algorithm is exactly TD$(0)$ with random sampling (samples are drawn i.i.d.~and not from a trajectory), and report a Markov design bound (the bound is computed only at the states used by the algorithm) of $O(\sqrt{\frac{d}{n\nu}})$ for LSTD$(0)$. All these results are for the on-policy setting, except the one by~\citet{Antos06learningnear-optimal} that also holds for the off-policy formulation. Another work in the off-policy setting is by~\citet{pires:2012:inverse} that uses a bounding trick and improves the result of~\citet{Antos06learningnear-optimal} by a $\log d$ factor. 

The line of research reported here has much in common with work on proximal reinforcement learning \citep{proximalrl}, which explores first-order reinforcement learning algorithms using {\em mirror maps} \citep{bubeck2014optml,juditsky2008solving} to construct primal-dual spaces. This work began originally with a dual space formulation of first-order sparse TD learning \citep{mahadevan:MID:2012}. A saddle point formulation for off-policy TD learning was initially explored in \citet{ROTD:NIPS2012}, where the objective function is the norm of the approximation residual of a linear inverse problem \citep{pires:2012:inverse}. A sparse off-policy GTD2 algorithm with regularized dual averaging is introduced by \citet{ZHIWEI2014}. These studies provide different approaches to formulating the problem, first as a variational inequality problem \citep{juditsky2008solving,proximalrl} or as a linear inverse problem \citep{ROTD:NIPS2012}, or as a quadratic objective function (MSPBE) using two-time-scale solvers \citep{ZHIWEI2014}. In this paper, we are going to explore the true nature of the GTD algorithms as stochastic gradient  algorithm w.r.t the convex-concave saddle-point formulations of NEU and MSPBE.


\section{SADDLE-POINT FORMULATION OF GTD ALGORITHMS}
\label{sec:saddle-point}

In this section, we show how the GTD and GTD2 algorithms can be formulated as true stochastic gradient (SG) algorithms by writing their respective objective functions, NEU and MSPBE, in the form of a convex-concave saddle-point. As discussed earlier, this new formulation of GTD and GTD2 as true SG methods allows us to use the convergence analysis techniques for SGs in order to derive finite-sample performance bounds for these RL algorithms. Moreover, it allows us to use more efficient algorithms that have been recently developed to solve SG problems, such as {\em stochastic Mirror-Prox} (SMP)~\citep{juditsky2008solving}, to derive more efficient versions of GTD and GTD2. 

A particular type of convex-concave saddle-point formulation is formally defined as  
\begin{equation}
\mathop {\min }\limits_\theta  \mathop {\max }\limits_y \big( {L(\theta ,y) = \left\langle {b-A\theta,y} \right\rangle  + F(\theta ) - K(y)} \big),
\label{eq:spgeneral}
\end{equation}
where $F(\theta)$ is a convex function and $K(y)$ is a smooth convex function such that  
\begin{equation}
\label{eq:saddle-gcondition}
K(y) - K(x) - \left\langle {\nabla K(x),y - x} \right\rangle  \le \frac{{{L_K}}}{2}||x - y|{|^2}.
\end{equation}
Next we follow~\citet{juditsky2008solving,RobustSA:2009,chen2013optimal} and define the following error function for the saddle-point problem~\eqref{eq:spgeneral}.
\begin{definition}
The error function of the saddle-point problem~\eqref{eq:spgeneral} at each point $(\theta',y')$ is defined as 
\begin{equation}
{\rm{Err}}(\theta',y') = \max_y\;L(\theta',y) - \min_\theta\;L(\theta,y').
\label{eq:errdef}
\end{equation}
\end{definition}

%

In this paper, we consider the saddle-point problem~\eqref{eq:spgeneral} with $F(\theta)=0$ and $K(y)=\frac{1}{2}||y||_M^2$, i.e.,
\begin{equation}
\mathop{\min}\limits_{\theta}\mathop{\max}\limits_{y}\Big({L(\theta,y)=\left\langle {b-A\theta,y}\right\rangle -\frac{1}{2}||y||_{M}^{2}}\Big),
\label{eq:sp}
\end{equation}
where $A$ and $b$ were defined by Eq.~\ref{eq:abc}, and $M$ is a positive definite matrix. It is easy to show that $K(y)=\frac{1}{2}||y||^2_M$ satisfies the condition in Eq.~\ref{eq:saddle-gcondition}. 

We first show in Proposition~\ref{pro:1} that if $(\theta^*,y^*)$ is the saddle-point of problem~\eqref{eq:sp}, then $\theta^*$ will be the optimum of NEU and MSPBE defined in Eq.~\ref{eq:j}. We then prove in Proposition~\ref{pro:2} that GTD and GTD2 in fact find this saddle-point.
%
\begin{proposition}
For any fixed $\theta$, we have $\frac{1}{2}J(\theta)=\mathop{\max}_yL(\theta,y)$, where $J(\theta)$ is defined by Eq.~\ref{eq:j}.
\label{pro:1}
\end{proposition}
\begin{proof}
Since $L(\theta,y)$ is an unconstrained quadratic program w.r.t.~$y$, the optimal $y^*(\theta)=\arg\max_y L(\theta,y)$ can be analytically computed as
\begin{equation}
y^{*}(\theta)=M^{-1}(b-A\theta).
\label{eq:ystar}
\end{equation}
The result follows by plugging $y^*$ into~\eqref{eq:sp} and using the definition of $J(\theta)$ in Eq.~\ref{eq:j} and Lemma~\ref{lem:e}.
\end{proof}

\begin{proposition}
\label{pro:2}
GTD and GTD2 are true stochastic gradient algorithms w.r.t.~the objective function $L(\theta,y)$ of the saddle-point problem~\eqref{eq:sp} with $M=I$ and $M=C={\Phi^{\top}}\Xi\Phi$ (the covariance matrix), respectively.
\end{proposition}
\begin{proof} 
It is easy to see that the gradient updates of the saddle-point problem~\eqref{eq:sp} (ascending in $y$ and descending in $\theta$) may be written as 
\begin{eqnarray}
\label{eq:sg}
{y_{t + 1}} &=& {y_t} + {\alpha _t}\left( b-{{A}{\theta _t}  - {M}{y_t}} \right),\,\\
\nonumber
{\theta _{t + 1}} &=& {\theta _t} + {\alpha _t}A^\top {y_t}.
\end{eqnarray}
We denote ${\hat M} := 1$ (resp. ${\hat M} := {\hat C}$) for GTD (resp. GTD2). 
We may obtain the update rules of GTD and GTD2 by replacing $A$, $b$, and $C$ in~\eqref{eq:sg} with their unbiased estimates $\hat A$, $\hat b$, and $\hat C$ from Eq.~\ref{eq:atbtct}, which completes the proof. 
\end{proof}

\section{FINITE-SAMPLE ANALYSIS}
\label{sec:analysis}

In this section, we provide a finite-sample analysis for a revised version of the GTD/GTD2 algorithms. We first describe the revised GTD algorithms in Section~\ref{subsec:revised} and then dedicate the rest of Section~\ref{sec:analysis} to their sample analysis. Note that from now on we use the $M$ matrix (and its unbiased estimate $\hat M_t$) to have a unified analysis for GTD and GTD2 algorithms. As described earlier, $M$ is replaced by the identity matrix $I$ in GTD and by the covariance matrix $C$ (and its unbiased estimate $\hat C_t$) in GTD2.


\subsection{THE REVISED GTD ALGORITHMS}
\label{subsec:revised}

The revised GTD algorithms that we analyze in this paper (see Algorithm~\ref{alg:pgtd2}) have three differences with the standard GTD algorithms of Eqs.~\ref{eq:gtd} and~\ref{eq:gtd2} (and Eq.~\ref{eq:sg}). {\bf 1)} We guarantee that the parameters $\theta$ and $y$ remain bounded by projecting them onto bounded convex feasible sets $\Theta$ and $Y$ defined in Assumption~\ref{ass:xyfeasible}. In Algorithm~\ref{alg:pgtd2}, we denote by $\Pi_\Theta$ and $\Pi_Y$, the projection into sets $\Theta$ and $Y$, respectively. This is standard in stochastic approximation algorithms and has been used in off-policy TD($\lambda$)~\citep{yu2012least} and actor-critic algorithms (e.g.,~\citealt{Bhatnagar09NA}). {\bf 2)} after $n$ iterations ($n$ is the number of training samples in $\mathcal{D}$), the algorithms return the weighted (by the step size) average of the parameters at all the $n$ iterations (see Eq.~\ref{eq:bartheta}). {\bf 3)} The step-size $\alpha_t$ is selected as described in the proof of Proposition~\ref{pro:hp} in the supplementary material. Note that this fixed step size of $O(1/\sqrt{n})$ is required for the high-probability bound in Proposition~\ref{pro:hp} (see~\citealt{RobustSA:2009} for more details).
\begin{algorithm}
\caption{Revised GTD Algorithms}
\label{alg:pgtd2} 
\begin{algorithmic}[1]
\FOR {$t=1,\ldots,n$}
\STATE Update parameters

\vspace{-0.175in}
\begin{small}
\begin{align}
\label{eq:pgtd2}
y_{t+1} &= \Pi_Y \Big(y_t + \alpha_t(\hat{b_t} - \hat{A}_t\theta_t - \hat{M}_ty_t)\Big) \nonumber \\
\theta_{t+1} &= \Pi_\Theta\Big(\theta_t + \alpha_t \hat{A}_t^\top y_t\Big)
\end{align}
\end{small}
\vspace{-0.2in}

\ENDFOR
\STATE OUTPUT

\vspace{-0.075in}
\begin{small}
\begin{equation}
{\bar \theta _n} := \frac{{\sum\nolimits_{t = 1}^n {{\alpha _t}{\theta _t}} }}{{\sum\nolimits_{t = 1}^n {{\alpha _t}} }}
\quad , \quad
{\bar y _n} := \frac{{\sum\nolimits_{t = 1}^n {{\alpha _t}{y _t}} }}{{\sum\nolimits_{t = 1}^n {{\alpha _t}} }}
\label{eq:bartheta}
\end{equation}
\end{small}


\end{algorithmic}
\end{algorithm}

\subsection{ASSUMPTIONS}

In this section, we make several assumptions on the MDP and basis functions that are used in our finite-sample analysis of the revised GTD algorithms. These assumptions are quite standard and are similar to those made in the prior work on  GTD algorithms~\citep{Sutton:GTD1:2008,FastGradient:2009,maei2011gradient} and those made in the analysis of SG algorithms~\citep{RobustSA:2009}.


\begin{assumption}
(\textbf{Feasibility Sets})
\label{ass:xyfeasible}
We define the bounded closed convex sets $\Theta\subset\mathbb{R}^d$ and $Y\subset\mathbb{R}^d$ as the feasible sets in Algorithm~\ref{alg:pgtd2}. We further assume that the saddle-point $(\theta^*,y^*)$ of the optimization problem~\eqref{eq:sp} belongs to $\Theta\times Y$. We also define $D_\theta:=\big[\max_{\theta\in\Theta}||\theta||_2^2-\min_{\theta\in\Theta}||\theta||_2^2\big]^{1/2}$, $D_y:=\big[\max_{y\in Y}||y||_2^2-\min_{y\in Y}||y||_2^2\big]^{1/2}$, and $R=\max\big\{\max_{\theta\in\Theta}||\theta||_2,\max_{y\in Y}||y||_2\big\}$. 
%
\end{assumption}

\begin{assumption}
(\textbf{Non-singularity})
We assume that the covariance matrix $C  = \mathbb{E}[{\phi_i}\phi_i^\top]$ and matrix $A=\mathbb{E}\big[\rho_i\phi_i(\Delta\phi_i)^\top\big]$ are non-singular.
\label{ass:C}
\end{assumption}

\begin{assumption} 
(\textbf{Boundedness})
\label{ass:bound}
Assume the features $(\phi_i,\phi^{'}_i)$ have uniformly bounded second moments. This together with the boundedness of features (by $L$) and importance weights (by $\rho_{\max}$) guarantees that the matrices $A$ and $C$, and vector $b$ are uniformly bounded. 
\end{assumption}

This assumption guarantees that for any $(\theta,y)\in\Theta\times Y$, the unbiased estimators of $b-A\theta-My$ and $A^\top y$, i.e.,
\begin{align}
\mathbb{E}[\hat{b}_t - \hat{A}_t\theta - \hat{M}_t y] &= b - A\theta  - My, \nonumber \\
\mathbb{E}[\hat{A}_t^\top y] &= {A^\top}y,
\end{align}
all have bounded variance, i.e.,
\begin{align}
\nonumber
\mathbb{E}\big[||\hat{b}_t - \hat{A}_t\theta - \hat{M}_t y - (b - A\theta  - My)|{|^2}\big] &\le \sigma _1^2, \nonumber \\
\mathbb{E}\big[||\hat{A}_t^\top y - {A^\top }y|{|^2}\big] &\le \sigma _2^2,
\label{eq:sigma123}
\end{align}
where $\sigma_1$ and $\sigma_2$ are non-negative constants. We further define 
\begin{equation}
\label{eq:sigma}
\sigma^2  = \sigma _1^2 + \sigma _2^2.
\end{equation}

Assumption~\ref{ass:bound} also gives us the following ``light-tail'' assumption. 
There exist constants ${M_{*,\theta }}$ and ${M_{*,y}}$ such that

\vspace{-0.15in}
\begin{small}
\begin{align}
\label{eq:lt}
\nonumber
&\mathbb{E}[\exp \{ \frac{{||{{\hat b}_t} - {{\hat A}_t}\theta  - {{\hat M}_t}y|{|^2}}}{{M_{_{*,\theta }}^2}}\} ]  \le \exp \{ 1\}, \\
&\mathbb{E}[\exp \{ \frac{{||\hat A_t^ \top y|{|^2}}}{{M_{_{*,y}}^2}}\} ] \le \exp \{ 1\}. 
\end{align}
\end{small}
\vspace{-0.15in}

This ``light-tail'' assumption is equivalent to the assumption in Eq.~3.16~in~\citet{RobustSA:2009} and is necessary for the high-probability bound of Proposition~\ref{pro:hp}. We will show how to compute ${M_{*,\theta }},{M_{*,y}}$ in the Appendix.

\subsection{FINITE-SAMPLE PERFORMANCE BOUNDS}
\label{subsec:FSB}

The finite-sample performance bounds that we derive for the GTD algorithms in this section are for the case that the training set $\mathcal{D}$ has been generated as discussed in Section~\ref{sec:preliminaries}. We further discriminate between the on-policy ($\pi=\pi_b$) and off-policy ($\pi\neq\pi_b$) scenarios. The sampling scheme used to generate $\mathcal{D}$, in which the first state of each tuple, $s_i$, is an i.i.d.~sample from a distribution $\xi$, also considered in the original GTD and GTD2 papers, for the analysis of these algorithms, and not in the experiments~\citep{Sutton:GTD1:2008,FastGradient:2009}. Another scenario that can motivate this sampling scheme is when we are given a set of high-dimensional data  generated either in an on-policy or off-policy manner, and $d$ is so large that the value function of the target policy cannot be computed using a least-squares method (that involves matrix inversion), and iterative techniques similar to GTD/GTD2 are required. 


We first derive a high-probability bound on the error function of the saddle-point problem~\eqref{eq:sp} at the GTD solution $(\bar{\theta}_n,\bar{y}_n)$. Before stating this result in Proposition~\ref{pro:hp}, we report the following lemma that is used in its proof.

\begin{lemma}
\label{lem:abbound}
The induced $\ell_2$-norm of matrix $A$ and the $\ell_2$-norm of vector $b$ are bounded by 
\begin{align}
||A||_2
 \le  (1 + \gamma ){\rho _{\max }}L^2 d,\quad\;
||b||_2
 \le  {\rho _{\max }}L{{\rm{R}}_{\max }}. 
 \label{eq:abbound}
\end{align}
\end{lemma}
\begin{proof}
See the supplementary material. 
\end{proof}

\begin{proposition}
\label{pro:hp}
Let $(\bar{\theta}_n,\bar{y}_n)$ be the output of the GTD algorithm after $n$ iterations (see Eq.~\ref{eq:bartheta}). Then, with probability at least $1-\delta$, we have 
\begin{align}
\label{eq:hp}
{\rm{Err}}({\bar \theta _n}&,{\bar y_n}) \le \sqrt {\frac{5}{n}} (8 + 2\log \frac{2}{\delta }){R^2} \\
&\times \left(\rho_{\max}L\Big(2(1 + \gamma )Ld + \frac{R_{\max}}{R}\Big) + \tau + \frac{\sigma }{R}\right), \nonumber
\end{align}
where ${\rm{Err}}({\bar \theta _n},{\bar y_n})$ is the error function of the saddle-point problem~\eqref{eq:sp} defined by Eq.~\ref{eq:errdef}, $R$ defined in Assumption~\ref{ass:xyfeasible}, $\sigma$ is from Eq.~\ref{eq:sigma}, and $\tau=\sigma_{\max}(M)$ is the largest singular value of $M$, which means $\tau= 1$ for GTD and $\tau=\sigma_{\max}(C)$ for GTD2.
\end{proposition}
\begin{proof}
See the supplementary material. 
\end{proof}

\begin{theorem}
\label{thm:1}
Let $\bar{\theta}_n$ be the output of the GTD algorithm after $n$ iterations (see Eq.~\ref{eq:bartheta}). Then, with probability at least $1-\delta$, we have 
\begin{equation}
\frac{1}{2}||A{\bar \theta _n} - b||_{\xi} ^2 \le \tau {\xi _{\max }}\;{\rm{Err}}({\bar \theta _n},{\bar y_n}).
\label{eq:thm1}
\end{equation}
%
%
\end{theorem}
\begin{proof}
From Proposition~\ref{pro:1}, for any $\theta$, we have 
\begin{equation*}
\max_{y}\;L(\theta,y) = \frac{1}{2}||A\theta  - b||_{M^{-1}}^2.
\end{equation*}
Given Assumption~\ref{ass:C}, the system of linear equations $A\theta=b$ has a solution $\theta^*$, i.e., the (off-policy) fixed-point $\theta^*$ exists, and thus, we may write 
\begin{align*}
\min_\theta\;\max_y\;L(\theta,y) &= \min_\theta\;\frac{1}{2}||A\theta - b||_{M^{-1}}^2 \\ 
&= \frac{1}{2}||A{\theta^*} - b||_{{M^{ - 1}}}^2 = 0.
\end{align*}
In this case, we also have\footnote{We may write the second inequality as an equality for our saddle-point problem defined by Eq.~\ref{eq:sp}.}
\begin{align}
\min_\theta\;L(\theta,y) &\le \max_y\;\min_\theta\;L(\theta,y) \le \min_\theta\;\max_y\;L(\theta,y) \nonumber \\
&= \frac{1}{2}||A\theta^* - b||_{M^{-1}}^2 = 0.
\label{eq:thm1-1}
\end{align}
From Eq.~\ref{eq:thm1-1}, for any $(\theta,y)\in\Theta\times Y$ including $(\bar{\theta}_n,\bar{y}_n)$, we may write
%
\begin{align}
\label{eq:err1}
{\rm{Err}}(\bar{\theta}_n,\bar{y}_n) &= \max_y\;L(\bar{\theta}_n,y) - \min_\theta\;L(\theta,\bar{y}_n) \\ 
&\ge \max_y\;L(\bar{\theta}_n,y) = \frac{1}{2}||A\bar{\theta}_n - b||_{M^{ - 1}}^2. \nonumber
\end{align}
Since $||A\bar{\theta}_n - b||_{\xi}^2 \le \tau \xi_{\max}\;||A\bar{\theta}_n - b||_{M^{ - 1}}^2$, where $\tau$ is the largest singular value of $M$, we have
\begin{equation}
\frac{1}{2}||A{\bar \theta _n} - b||_{\xi} ^2 \le \frac{{\tau {\xi _{\max }}}}{2}||A{\bar \theta _n} - b||_{{M^{ - 1}}}^2 \le \tau {\xi _{\max }}\;{\rm{Err}}({\bar \theta _n},{\bar y_n}).
\label{eq:err3} 
\end{equation}
The proof follows by combining Eqs.~\ref{eq:err3} and Proposition~\ref{pro:hp}.  It completes the proof.
%
\end{proof}

With the results of Proposition~\ref{pro:hp} and Theorem~\ref{thm:1}, we are now ready to derive finite-sample bounds on the performance of GTD/GTD2 in both on-policy and off-policy settings. 



\subsubsection{On-Policy Performance Bound}

In this section, we consider the on-policy setting in which the behavior and target policies are equal, i.e.,~$\pi_b=\pi$, and the sampling distribution $\xi$ is the stationary distribution of the target policy $\pi$ (and the behavior policy $\pi_b$). We use Lemma~\ref{lem:v} to derive our on-policy bound. The proof of this lemma can be found in~\citet{DantzigRL:2012}.

\begin{lemma}
\label{lem:v}
For any parameter vector $\theta$ and corresponding $\hat v = \Phi \theta $, the following equality holds 

\vspace{-0.2in}
\begin{small}
\begin{equation}
V - \hat v = {(I - \gamma \Pi P)^{ - 1}}\left[ {\left( {V - \Pi V} \right) + \Phi {C^{ - 1}}(b - A\theta )} \right].
\label{eq:vbound}
\end{equation}
\end{small}
\vspace{-0.25in}
\end{lemma}
%
%

Using Lemma~\ref{lem:v}, we derive the following performance bound for GTD/GTD2 in the on-policy setting.

\begin{proposition}
\label{pro:4}
Let $V$ be the value of the target policy and ${{\bar v}_n} = \Phi {{\bar \theta }_n}$, where $\bar{\theta}_n$ defined by~\eqref{eq:bartheta}, be the value function returned by on-policy GTD/GTD2. Then, with probability at least $1-\delta$, we have

\vspace{-0.15in}
\begin{small}
\begin{equation}
||V - {\bar v_n}|{|_\xi } \le \frac{1}{{1 - \gamma }}\left( {||V - \Pi V|{|_\xi } + \frac{L}{\nu }\sqrt {2d\tau {\xi _{\max }}{\rm{Err}}({{\bar \theta }_n},{{\bar y}_n})} } \right)
\label{eq:pro4}
\end{equation}
\end{small}
\vspace{-0.15in}

where $\rm{Err}(\bar \theta_n,\bar y_n)$ is upper-bounded by Eq.~\ref{eq:hp} in Proposition~\ref{pro:hp}, with $\rho_{\max}=1$ (on-policy setting).
%
\end{proposition}
\begin{proof}
See the supplementary material. 
\end{proof}
\textbf{Remark:}
It is important to note that Proposition~\ref{pro:4} shows that the error in the performance of the GTD/GTD2 algorithm in the on-policy setting is of $O\left( {\frac{{{L^2}d\sqrt {\tau {\xi _{\max }}\log \frac{1}{\delta }} }}{{{n^{1/4}\nu}}}} \right)$. Also note that the term $\frac{\tau}{\nu}$ in the GTD2 bound is the conditioning number of the covariance matrix $C$.




\subsubsection{Off-Policy Performance Bound}

In this section, we consider the off-policy setting in which the behavior and target policies are different, i.e.,~$\pi_b\neq\pi$, and the sampling distribution $\xi$ is the stationary distribution of the behavior policy $\pi_b$. We assume that off-policy fixed-point solution exists, i.e.,~there exists a $\theta^*$ satisfying $A\theta^*=b$. Note that this is a direct consequence of Assumption~\ref{ass:C} in which we assumed that the matrix $A$ in the off-policy setting is non-singular. We use Lemma~\ref{lem:kolter} to derive our off-policy bound. The proof of this lemma can be found in~\citet{Kolter:offpolicyTD}. Note that $\kappa (\bar D)$ in his proof is equal to $\sqrt{\rho_{\max}}$ in our paper.


\begin{lemma}
\label{lem:kolter}
If $\Xi$ satisfies the following linear matrix inequality 
\begin{align}
\left[ {\begin{array}{*{20}{c}}
{{\Phi ^\top}\Xi \Phi }&{{\Phi ^\top}\Xi P\Phi }\\
{{\Phi ^\top}{P^\top}\Xi \Phi }&{{\Phi ^\top}\Xi \Phi }
\end{array}} \right] \succeq 0
\label{eq:kolterlmi}
\end{align}
and let $\theta^*$ be the solution to $A\theta^* = b$, 
 then we have
\begin{align}
||V - \Phi \theta^* |{|_{\xi} } \le \frac{{1 + \gamma \sqrt {{\rho _{\max }}} }}{{1 - \gamma }}||V - \Pi V|{|_{\xi} }.
\label{eq:kolter2}
\end{align}
\end{lemma}


Note that the condition on $\Xi$ in Eq.~\ref{eq:kolterlmi} guarantees that the behavior and target policies are not too far away from each other. Using Lemma~\ref{lem:kolter}, we derive the following performance bound for GTD/GTD2 in the off-policy setting.

\begin{proposition}
\label{pro:5}
Let $V$ be the value of the target policy and ${{\bar v}_n} = \Phi {{\bar \theta }_n}$, where $\bar{\theta}_n$ is defined by~\eqref{eq:bartheta}, be the value function returned by off-policy GTD/GTD2. Also let the sampling distribution $\Xi$ satisfies the condition in Eq.~\ref{eq:kolterlmi}. Then, with probability at least $1-\delta$, we have
%
%
\begin{align}
\label{eq:pro5}
||V - {\bar v_n}|{|_\xi } & \le \frac{{1 + \gamma \sqrt {{\rho _{\max }}} }}{{1 - \gamma }}||V - \Pi V|{|_\xi }\\
\nonumber
&+ \sqrt {\frac{{2{\tau _C}\tau {\xi _{\max }}}}{{{\sigma _{\min }}({A^ \top }{M^{ - 1}}A)}}{\rm{Err}}({{\bar \theta }_n},{{\bar y}_n})},
\end{align}
where ${\tau _C} = \sigma_{\max} (C)$. 
\end{proposition}
\begin{proof}
See the supplementary material. 
\end{proof}

\section{ACCELERATED ALGORITHM}

As discussed at the beginning of Section~\ref{sec:saddle-point}, this saddle-point formulation not only gives us the opportunity to use the techniques for the analysis of SG methods to derive finite-sample performance bounds for the GTD algorithms, as we will show in Section~\ref{sec:analysis}, but also it allows us to use the powerful algorithms that have been recently developed to solve the SG problems and derive more efficient versions of GTD and GTD2. Stochastic Mirror-Prox (SMP)~\citep{juditsky2008solving} is an ``almost dimension-free'' non-Euclidean extra-gradient method that deals with both smooth and non-smooth stochastic optimization problems (see~\citealt{sra2011optimization} and~\citealt{bubeck2014optml} for more details). Using SMP, we propose a new version of GTD/GTD2, called GTD-MP/GTD2-MP, with the following update formula:\footnote{For simplicity, we only describe mirror-prox GTD methods where the mirror map is identity, which can also be viewed as extragradient (EG) GTD methods. \cite{proximalrl} gives a more detailed discussion of a  broad range of mirror maps in RL.} 

\vspace{-0.15in}
\begin{small}
\begin{align*}
y_t^m &= y_t + \alpha _t(\hat b_t - \hat  A_t\theta _t - \hat  M_ty_t), \quad\;\;\; \theta_t^m = \theta_t + \alpha_t \hat A_t^\top y_t, \\
y_{t + 1} &= y_t + \alpha_t(\hat b_t - \hat A_t\theta_t^m - \hat M_ty_t^m), \; \theta_{t + 1} = \theta_t + \alpha_t \hat A_t^\top y_t^m.
\end{align*}
\end{small}
\vspace{-0.2in}

After $T$ iterations, these algorithms return ${\bar \theta _T}: = \frac{{\sum\nolimits_{t = 1}^T {{\alpha _t}{\theta _t}} }}{{\sum\nolimits_{t = 1}^T {{\alpha _t}} }}$ and ${\bar y_T}: = \frac{{\sum\nolimits_{t = 1}^T {{\alpha _t}{y_t}} }}{{\sum\nolimits_{t = 1}^T {{\alpha _t}} }}$. 
The details of the algorithm is shown in Algorithm~\ref{alg:GTD2mp}, and the experimental comparison study between GTD2 and GTD2-MP is reported in Section~\ref{sec:exp}.

\begin{algorithm}
\caption{GTD2-MP}
\label{alg:GTD2mp} 
\begin{algorithmic}[1]
\FOR {$t=1,\ldots,n$}
\STATE Update parameters

\vspace{-0.175in}
\begin{small}
\begin{align*}
\nonumber
{\delta_{t}} &= {r_t} - \theta _t^\top \Delta {\phi _t}\\
\nonumber
y_t^m &= {y_t} + {\alpha _t}(\rho_t{\delta _t} - \phi _t^\top {y_t}){\phi _t}\\
\nonumber
\theta _t^m &= {\theta _t} + {\alpha _t}\rho_t\Delta {\phi _t}(\phi _t^\top {y_t})\\
\nonumber
\delta _t^m &= {r_t} - (\theta _t^m)^\top \Delta {\phi _t}\\
\nonumber
{y_{t + 1}} &= {y_t} + {\alpha _t}(\rho_t\delta _t^m - \phi _t^\top y_t^m){\phi _t}\\
\nonumber
{\theta _{t + 1}} &= {\theta _t} + {\alpha _t}\rho_t\Delta {\phi _t}(\phi _t^\top y_t^m)
\end{align*}
\end{small}
\vspace{-0.2in}

\ENDFOR
\STATE OUTPUT

\vspace{-0.075in}
\begin{small}
\begin{equation}
{\bar \theta _n} := \frac{{\sum\nolimits_{t = 1}^n {{\alpha _t}{\theta _t}} }}{{\sum\nolimits_{t = 1}^n {{\alpha _t}} }}
\quad , \quad
{\bar y _n} := \frac{{\sum\nolimits_{t = 1}^n {{\alpha _t}{y _t}} }}{{\sum\nolimits_{t = 1}^n {{\alpha _t}} }}
\end{equation}
\end{small}
\vspace{-0.1in}

\end{algorithmic}
\end{algorithm}


\section{FURTHER ANALYSIS}

\subsection{ACCELERATION ANALYSIS}
In this section, we are going to discuss the convergence rate of the accelerated algorithms using off-the-shelf accelerated solvers for saddle-point problems. For simplicity, we will discuss the error bound of $\frac{1}{2}||A\theta  - b||_{{M^{ - 1}}}^2$, and the corresponding error bound of $\frac{1}{2}||A\theta  - b||_{\xi}^2$ and $\|V - {\bar v}_n|{|_{\xi} }$ can be likewise derived as in above analysis.
As can be seen from the above analysis, the convergence rate of the GTD algorithms family is
%
\begin{equation}
({\bf{GTD}}/{\bf{GTD2}}):\quad O\left( {\frac{{\tau  + ||A|{|_2}  + \sigma }}{{\sqrt n }}} \right)
\end{equation}
In this section, we raise an interesting question: what is the ``optimal" GTD algorithm? To answer this question, we review the convex-concave formulation of GTD2. 
According to convex programming complexity theory~\citep{juditsky2008solving},  the un-improvable convergence rate of stochastic saddle-point problem (\ref{eq:sp}) is 
%
%

\begin{equation}
({\bf{Optimal}}):\quad O\left( {\frac{{\tau}}{{{n^2}}} + \frac{{||A|{|_2}}}{n} + \frac{\sigma }{{\sqrt n }}} \right)
\label{eq:rateoptimal}
\end{equation}
There are many readily available stochastic saddle-point solvers, such as stochastic Mirror-Prox  (GTD2-MP)~\citep{juditsky2008solving} algorithm, which leads to our proposed GTD2-MP algorithm.  SMP is able to accelerate the convergence rate of our gradient TD method to: 
%
\begin{equation}
({\bf{SMP}}):\quad O\left( {\frac{{\tau  + ||A|{|_2} }}{n} + \frac{\sigma }{{\sqrt n }}} \right),
\end{equation}
 and stochastic accelerated primal-dual (SAPD) method \citep{chen2013optimal} which can reach the optimal convergence rate in (\ref{eq:rateoptimal}). Due to space limitations, we are unable to present a more complete description, and refer interested readers to \citet{juditsky2008solving, chen2013optimal} for more details.


\subsection{LEARNING WITH BIASED $\rho_t$}

The importance weight factor $\rho_t$ is lower bounded by $0$, but yet may have an arbitrarily large upper bound.
In real applications, the importance weight factor $\rho_t$ may not be estimated exactly, i.e., the estimation $\hat{\rho}_t$ is a biased estimation of the true $\rho_t$.  To this end, the stochastic gradient we obtained is not the unbiased gradient of $L(\theta,y)$ anymore.
  This falls into a broad category of learning with inexact stochastic gradient, or termed as stochastic gradient methods with an inexact oracle \citep{inexact:stochastic:devolder2011}. Given the inexact stochastic gradient, the convergence rate and performance bound become much worse than the results with exact stochastic gradient. Based on the analysis by \citet{juditsky2008solving}, we have the error bound for inexact estimation of $\rho_t$.
\begin{proposition}
Let ${{\bar{\theta}  }_n}$ be defined as above. Assume at the $t$-th iteration, $\hat{\rho}_t$ is the estimation of the importance weight factor $\rho_t$ with bounded bias such that 
$
\mathbb{E}[ \hat{\rho}_t  - {\rho _t} ] \le \epsilon. 
$ 
The convergence rates of GTD/GTD2 algorithms with iterative averaging is as follows,
i.e.,
\begin{equation}
 ||A{\bar \theta _n} - b||_{{M^{ - 1}}}^2 \le O\left( {\frac{{\tau  + ||A||_2  + \sigma }}{{\sqrt n }}} \right) +O(\epsilon)
\label{eq:biasrho}
\end{equation}
\end{proposition}
This implies that the inexact estimation of $\rho_t$ may  cause disastrous estimation error, which implies that an exact estimation of $\rho_t$ is very important.


\subsection{FINITE-SAMPLE ANALYSIS OF ONLINE LEARNING}
Another more challenging scenario is online learning scenario, where the samples are interactively generated by the environment, or by an interactive agent. The difficulty lies in that the sample distribution does not follow i.i.d sampling condition anymore, but follows an underlying Markov chain $\mathcal{M}$. If the Markov chain $\mathcal{M}$'s mixing time is small enough, i.e., the sample distribution reduces to the stationary distribution of $\pi_b$ very fast, our analysis still applies. However, it is usually the case that the underlying Markov chain's mixing time $\tau_{\rm{mix}}$ is not small enough. The analysis result can be obtained by extending the result of recent work~\citep{duchi2012ergodic}  from  strongly convex loss functions to saddle-point problems, which is non-trivial and is thus left for future work.


\subsection{DISCUSSION OF TDC ALGORITHM}

Now we discuss the limitation of our analysis with regard to the temporal difference with correction (TDC) algorithm \citep{FastGradient:2009}. 
Interestingly, the TDC algorithm seems not to have an explicit saddle-point representation, since it incorporates the information of the optimal $y_{t}^{*}(\theta_{t})$ into the update of $\theta_t$, a quasi-stationary condition which is commonly used in two-time-scale stochastic approximation approaches.
 An intuitive answer to the advantage of TDC over GTD2 is that
the TDC update of $\theta_{t}$ can be considered as incorporating
the prior knowledge into the update rule: for a stationary $\theta_{t}$,
if the optimal  $y_{t}^{*}(\theta_{t})$
has a closed-form solution or is easy to compute, then incorporating
this $y_{t}^{*}(\theta_{t})$ into the update law tends to accelerate
the algorithm's convergence performance. For the GTD2 update, note that there is a sum of two terms where $y_{t}$
appears, which are $\rho_t({\phi_{t}}-\gamma\phi_{t}^{\prime})(y_{t}^{T}{\phi_{t}})={\rho_t\phi_{t}}(y_{t}^{T}{\phi_{t}})-\gamma\rho_t\phi_{t}^{\prime}(y_{t}^{T}{\phi_{t}})$. Replacing $y_{t}$ in the first term with $y^*_{t}(\theta_t)=\mathbb{E}{[{\phi_{t}}\phi_{_{t}}^{T}]^{-1}}\mathbb{E}[\rho_t{\delta_{t}}(\theta_t){\phi_{t}}]$,
we have the TDC update rule. 
%
%
Note that in contrast to GTD/GTD2, TDC is a two-time scale algorithm; Also, note that TDC does not minimize \emph{any} objective functions and the convergence of TDC requires more restrictions than GTD2 as shown by \citet{FastGradient:2009}.


\section{EMPIRICAL EVALUATION}
\label{sec:exp}
In this section, we compare the previous GTD2 method with our proposed GTD2-MP method using various domains with regard to their value function approximation performance capability. It should be mentioned that since the major focus of this paper is on policy evaluation, the comparative study focuses on value function approximation and thus comparisons on control learning performance is not reported in this paper. 

\subsection{BAIRD DOMAIN}
The Baird example \citep{Baird:ResidualAlgorithms1995} is a well-known example to test the performance of off-policy convergent algorithms. 
Constant stepsize $\alpha = 0.005$ for GTD2 and $\alpha = 0.004$ for GTD2-MP, which are   chosen via comparison studies as in \citep{dann2014tdsurvey}.
Figure~\ref{fig:star} shows the MSPBE curve of GTD2, GTD2-MP of $8000$ steps averaged over $200$ runs. We can see that GTD2-MP has a significant improvement over the GTD2 algorithm
wherein both the MSPBE and the variance are substantially reduced.
\begin{figure}[h]
\centering
\includegraphics[width=.5\textwidth,height=1.75in]{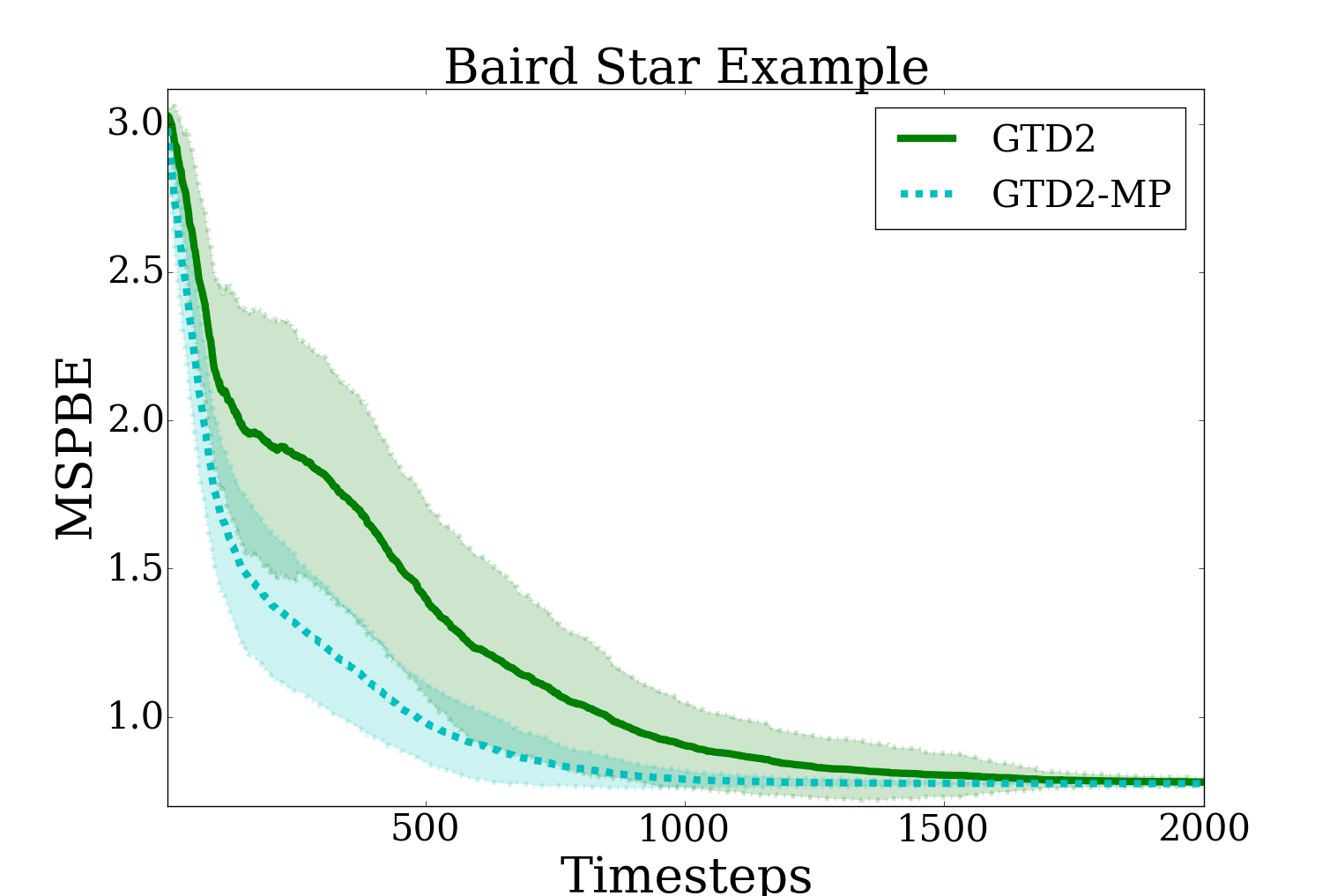}
\caption{Off-Policy Convergence Comparison}
\label{fig:star}
\end{figure}

%


\subsection{$50$-STATE CHAIN DOMAIN}

The 50 state chain~\citep{lagoudakis:jmlr} is a standard MDP domain. There are 50 discrete
states $\{ {s_i}\} _{i = 1}^{50}$ and two actions moving the agent left ${s_i} \to {s_{\max ({{i - 1}},1)}}$ and right ${s_i} \to {s_{\min ({{i + 1}},50)}}$. The actions succeed with probability $0.9$; failed actions move the agent in the opposite direction. The discount factor is $\gamma = 0.9$. The agent receives a reward of $+1$ when in states $s_{10}$ and $s_{41}$. All other states have a reward of $0$.
In this experiment, we compare the performance of the value approximation w.r.t different set of stepsizes $\alpha  = 0.0001,0.001,0.01,0.1,0.2, \cdots ,0.9$ using the BEBF basis \citep{bebf:parr:icml07}, and Figure~\ref{fig:chain} shows the value function approximation result, where the cyan curve is the true value function, the red dashed curve is the GTD result,and the black curve is the GTD2-MP result. From the figure, one can see that GTD2-MP is much more robust with stepsize choice than the GTD2 algorithm.

\begin{figure}
\centering
\includegraphics[width=0.5\textwidth, height = 6cm]{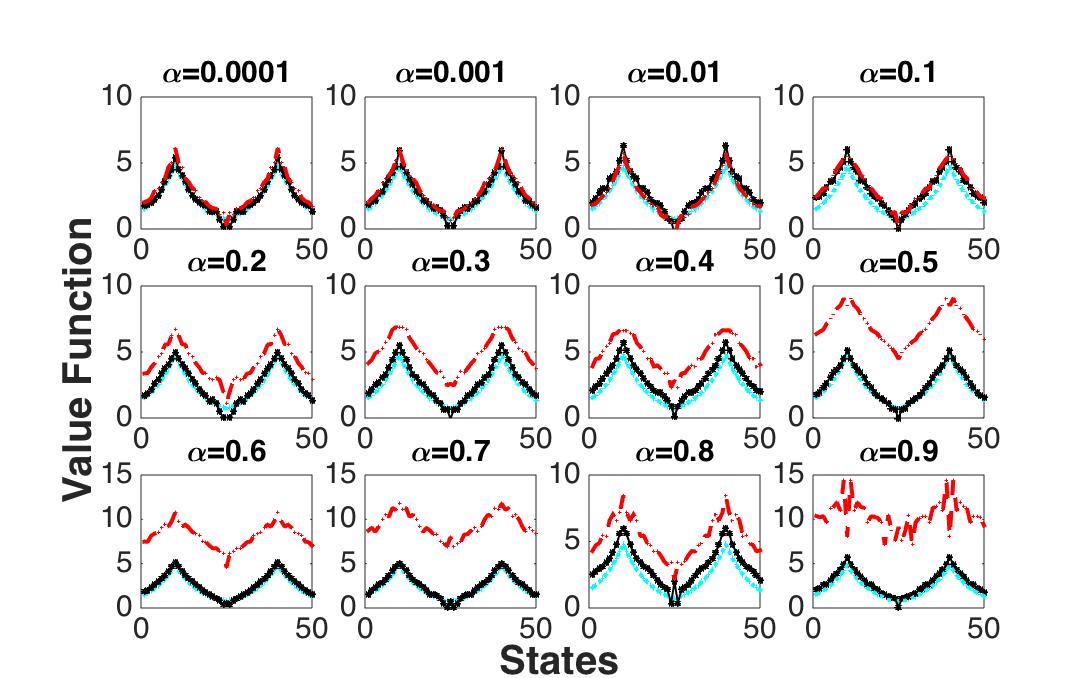}
\caption{Chain Domain}
\label{fig:chain}
\end{figure}


\subsection{ENERGY MANAGEMENT DOMAIN}
In this experiment we compare the performance of the algorithms on an energy management domain.
The decision maker must decide how much energy to purchase or sell subject to stochastic prices. This problem is relevant in the context of utilities as well as in settings such as hybrid vehicles. The prices are generated from a Markov chain process. The amount of available storage is limited and it also degrades with use. The degradation process is based on the physical properties of lithium-ion batteries and discourages fully charging or discharging the battery. The energy arbitrage problem is closely related to the broad class of inventory management problems, with the storage level corresponding to the inventory. However, there are no known results describing the structure of optimal threshold policies in energy storage.

Note that since for this off-policy evaluation problem, the formulated $A\theta=b$ does not have a solution, and thus the optimal MSPBE($\theta^*$) (resp. MSBE($\theta^*$) ) do not reduce to $0$. 
 The result is averaged over $200$ runs,  and $\alpha = 0.001$ for both GTD2 and GTD2-MP is chosen via comparison studies for each algorithm.
As can be seen from FIgure~\ref{fig:inv}, in the initial transit state, GTD2-MP performs much better than  GTD2 at the transient state. Then after reaching the steady state, as can be seen from Table~\ref{tab:inv}, we can see that GTD2-MP reaches better steady state solution than the GTD algorithm.
Based on the above empirical results and many other experiments we have conducted in other domains, we can conclude that GTD2-MP usually performs much better than the ``vanilla'' GTD2 algorithm. 

\begin{figure}
\centering
\begin{minipage}{1\textwidth}
\includegraphics[width= .5\textwidth, height=1.3in]{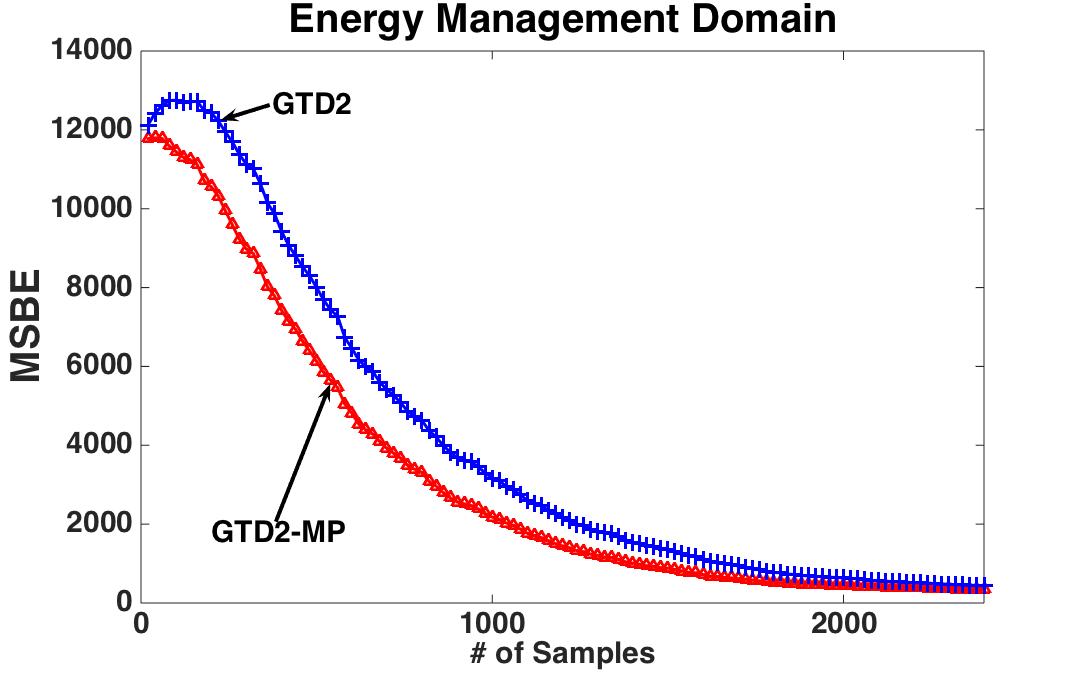}\\
\includegraphics[width= .5\textwidth, height=1.3in]{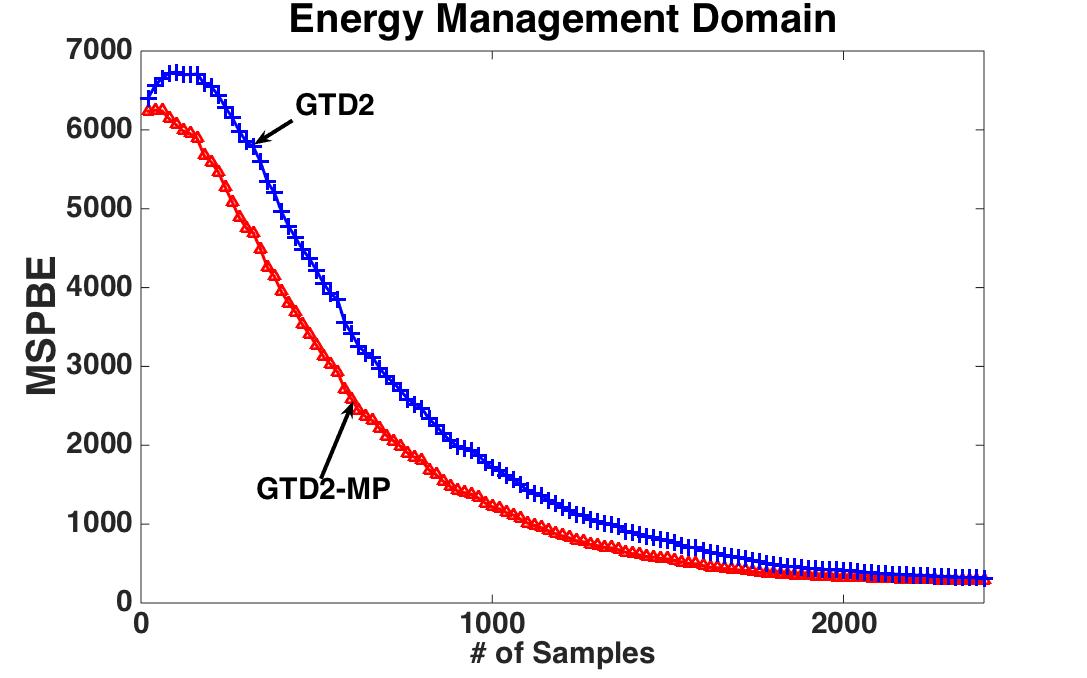}
\end{minipage}
\caption{Energy Management Example}
\label{fig:inv}
\end{figure}

\begin{table}
\centering
\begin{tabular}{|c|c|c|}
\hline 
Algorithm & MSPBE & MSBE        \tabularnewline
\hline 
GTD2 & $176.4$ & $228.7$      \tabularnewline
\hline 
GTD2-MP & $138.6$ & $191.4$    \tabularnewline
\hline 
\end{tabular}
\caption{Steady State Performance Comparison}
\label{tab:inv}
\end{table}



\section{SUMMARY}

In this paper, we showed how gradient TD methods  can be shown to be true stochastic gradient methods with respect to a saddle-point primal-dual objective function, which paved the way for the finite-sample analysis of off-policy convergent gradient-based temporal difference learning algorithms such as GTD and GTD2. 
 Both error bound and performance bound are provided, which shows that the value function approximation bound of the GTD algorithms family is $O\left( {\frac{d}{{{n^{1/4}}}}} \right)$. 
 Further, two revised algorithms, namely the projected GTD2 algorithm and the accelerated GTD2-MP algorithm,  are proposed. 
 There are many interesting directions for future research. Our framework can be easily used to design regularized sparse gradient off-policy TD methods. 
 One interesting direction is to investigate the convergence rate and performance bound for the TDC algorithm, which lacks a saddle-point formulation. 
 The other is to explore  tighter value function approximation bounds for off-policy learning. 

\section*{Acknowledgements}
This material is based upon work supported by the  National Science Foundation under Grant Nos. IIS-1216467. Any opinions, findings, and conclusions or recommendations expressed in this material are those of the authors and do not necessarily reflect the views of the NSF.

\newpage
\begin{small}
\bibliography{thesisbib}
\bibliographystyle{named}
\end{small}


\newpage
\appendix

\section{PROOF OF LEMMA~\ref{lem:abbound}}
\begin{proof}
From the boundedness of the features (by $L$) and the rewards (by ${{\rm{R}}_{\max }}$), we have
\begin{align*}
||A||_2 &= ||\mathbb{E}[{\rho _t}{\phi _t}\Delta \phi _t^ \top ]|{|_2} \\
&\le {{\max }_s}||\rho (s)\phi (s){{(\Delta \phi (s))}^ \top }|{|_2} \\
&\le {\rho _{\max }}{{\max }_s}||\phi (s)|{|_2}{{\max }_s}||\phi (s) - \gamma \phi '(s)|{|_2} \\
&\le {\rho _{\max }}{{\max }_s}||\phi (s)|{|_2}{{\max }_s}\left( {||\phi (s)|{|_2} + \gamma ||\phi '(s)|{|_2}} \right) \\
&\le (1 + \gamma ){\rho _{\max }}L^2 d.
\end{align*}
The second inequality is obtained by the consistent inequality of matrix norm, the third inequality comes from the triangular norm inequality, and the fourth inequality comes from the vector norm inequality $||\phi (s)|{|_2} \le ||\phi (s)|{|_\infty }\sqrt d  \le L\sqrt d$. The bound on $||b||_2$ can be derived in a similar way as follows.
\begin{align*}
||b||_2 &= ||\mathbb{E}[{\rho _t}{\phi _t}r_t ]|{|_2} \\
&\le {\max _s}||\rho (s)\phi (s)r(s)|{|_2} \\
&\le {\rho _{\max }}{{\max }_s}||\phi (s)|{|_2}{{\max }_s}||r (s)|{|_2} \\
&\le {\rho _{\max }}L{{\rm{R}}_{\max }}.
\end{align*}
It completes the proof.
\end{proof}

\section{PROOF OF PROPOSITION~\ref{pro:hp}}

\begin{proof}
The proof of Proposition~\ref{pro:hp} mainly relies on Proposition~3.2 in~\citet{RobustSA:2009}. We just need to map our convex-concave {\em stochastic} saddle-point problem in Eq.~\ref{eq:sp}, i.e., 
\begin{equation*}
\mathop{\min}\limits_{\theta\in\Theta}\mathop{\max}\limits_{y\in Y}\left({L(\theta,y)=\left\langle {b-A\theta,y}\right\rangle -\frac{1}{2}||y||_{M}^{2}}\right)
\end{equation*}
to the one in Section~3 of~\citet{RobustSA:2009} and show that it satisfies all the conditions necessary for their Proposition~3.2. Assumption~\ref{ass:xyfeasible} guarantees that our feasible sets $\Theta$ and $Y$ satisfy the conditions in~\citet{RobustSA:2009}, as they are non-empty bounded closed convex subsets of $\mathbb{R}^d$. We also see that our objective function $L(\theta,y)$ is {\em convex} in $\theta\in\Theta$ and {\em concave} in $y\in Y$, and also {\em Lipschitz continuous} on $\Theta\times Y$. It is known that in the above setting, our saddle-point problem in Eq.~\ref{eq:sp} is solvable, i.e.,~the corresponding {\em primal} and {\em dual} optimization problems: $\min_{\theta\in\Theta}\big[\max_{y\in Y}L(\theta,y)\big]$ and $\max_{y\in Y}\big[\min_{\theta\in\Theta}L(\theta,y)\big]$ are solvable with equal optimal values, denoted $L^*$, and pairs $(\theta^*,y^*)$ of optimal solutions to the respective problems from the set of saddle points of $L(\theta,y)$ on $\Theta\times Y$. 

For our problem, the {\em stochastic sub-gradient vector} $G$ is defined as

\vspace{-0.15in}
\begin{small}
\begin{equation*}
G(\theta ,y) = \left[ {\begin{array}{*{20}{c}}
{{G_\theta }(\theta ,y)}\\
{ - {G_y}(\theta ,y)}
\end{array}} \right] 
= \left[ {\begin{array}{*{20}{c}}
{ - \hat{A}_t^ \top y}\\
{ - (\hat{b}_t - \hat{A}_t\theta  - \hat{M}_ty)}
\end{array}} \right]. 
\end{equation*}
\end{small}
\vspace{-0.2in}

This guarantees that the {\em deterministic sub-gradient vector}

\vspace{-0.1in}
\begin{small}
\begin{equation*}
g(\theta ,y) = \left[ {\begin{array}{*{20}{c}}
{{g_\theta }(\theta ,y)}\\
{ - {g_y}(\theta ,y)}
\end{array}} \right] 
= \left[ {\begin{array}{*{20}{c}}
\mathbb{E}\big[G_\theta(\theta,y)\big] \\
- \mathbb{E}\big[G_y(\theta,y)\big]
\end{array}} \right] 
\end{equation*}
\end{small}
\vspace{-0.15in}

is well-defined, i.e.,~$g_\theta(\theta,y)\in\partial_\theta L(\theta,y)$ and $g_y(\theta,y)\in\partial_y L(\theta,y)$.

We also consider the Euclidean stochastic approximation (E-SA) setting in~\citet{RobustSA:2009} in which the {\em distance generating functions} $\omega_\theta:\Theta\rightarrow\mathbb{R}$ and $\omega_y:Y\rightarrow\mathbb{R}$ are simply defined as
\begin{equation*}
\omega_\theta=\frac{1}{2}||\theta||_2^2, \quad\quad \omega_y=\frac{1}{2}||y||_2^2,
\end{equation*}
modulus $1$ w.r.t.~$||\cdot||_2$, and thus, $\Theta^o=\Theta$ and $Y^o=Y$ (see pp.~1581~and~1582~in~\citealt{RobustSA:2009}). This allows us to equip the set $Z=\Theta\times Y$ with the distance generating function 
\begin{equation*}
\omega(z)=\frac{\omega_\theta(\theta)}{2D_\theta^2}+\frac{\omega_y(y)}{2D_y^2},
\end{equation*}
where $D_\theta$ and $D_y$ defined in Assumption~\ref{ass:xyfeasible}. 

Now that we consider the Euclidean case and set the norms to $\ell_2$-norm, we can compute upper-bounds on the expectation of the dual norm of the stochastic sub-gradients 
\begin{equation*}
\mathbb{E}\left[||G_\theta(\theta,y)||^2_{*,\theta}\right] \le M_{*,\theta}^2, \quad \mathbb{E}\left[||G_y(\theta,y)||^2_{*,y}\right] \le M_{*,y}^2,
\end{equation*}
where $||\cdot||_{*,\theta}$ and $||\cdot||_{*,y}$ are the dual norms in $\Theta$ and $Y$, respectively. Since we are in the Euclidean setting and use the $\ell_2$-norm, the dual norms are also $\ell_2$-norm, and thus, to compute $M_{*,\theta}$, we need to upper-bound $\mathbb{E}\left[||G_\theta(\theta,y)||_2^2\right]$ and $\mathbb{E}\left[||G_y(\theta,y)||_2^2\right]$. 

To bound these two quantities, we use the following equality that holds for any random variable $x$: 
\begin{equation*}
\mathbb{E}[||x||_2^2] =  \mathbb{E}[||x - {\mu _x}||_2^2] + ||{\mu _x}||_2^2,
\end{equation*}
where $\mu_x = \mathbb{E}[x]$. Here how we bound $\mathbb{E}\left[||G_\theta(\theta,y)||_2^2\right]$,

\vspace{-0.15in} 
\begin{small}
\begin{align*}
\nonumber
\mathbb{E}\left[||G_\theta(\theta,y)||_2^2\right] 
&= \mathbb{E}[||\hat A_t^ \top y|{|}^2_2] \\
\nonumber
&= \mathbb{E}[||\hat A_t^ \top y - {A^ \top }y|{|}_2^2] + ||{A^ \top }y|{|}^2_2 \\
\nonumber
&\le {\sigma }^2_2 + {(||A|{|_2}||y|{|_2})^2}\\
&\le {\sigma }^2_2 + ||A|{|}^2_2{R^2},
\end{align*}
\end{small}
\vspace{-0.15in} 

where the first inequality is from the definition of $\sigma_3$ in Eq.~\ref{eq:sigma123} and the consistent inequality of the matrix norm, and the second inequality comes from the boundedness of the feasible sets in Assumption~\ref{ass:xyfeasible}. Similarly we bound $\mathbb{E}\left[||G_y(\theta,y)||_2^2\right]$ as follows:

\vspace{-0.15in} 
\begin{small}
\begin{align*}
\mathbb{E}[||G_y(\theta,y&)||_2^2] = \mathbb{E}[||{{\hat b}_t} - {{\hat A}_t}\theta  - {{\hat M}_t}y||^2_2] \\
&= ||b - A\theta  + My||_2^2 \\
&+ \mathbb{E}[||{{\hat b}_t} - {{\hat A}_t}\theta  - {{\hat M}_t}y - (b - A\theta  - My)||_2^2]\\
&\le {(||b|{|_2} + ||A|{|_2}||\theta |{|_2} + \tau ||y|{|_2})^2} + \sigma _1^2\\
&\le \big(||b||_2 + (||A||_2 + \tau)R\big)^2 + \sigma _1^2,
\end{align*}
\end{small}
\vspace{-0.15in} 

where these inequalities come from the definition of $\sigma_1$ in Eq.~\ref{eq:sigma123} and the boundedness of the feasible sets in Assumption~\ref{ass:xyfeasible}. This means that in our case we can compute $M_{*,\theta }^2, M_{*,y}^2$ as
\begin{align*}
M_{*,\theta }^2 &= \sigma _2^2 + ||A||_2^2{R^2},\\
M_{*,y}^2 &= \big(||b||_2 + (||A||_2 + \tau)R\big)^2 + \sigma _1^2,
\end{align*}
and as a result
\begin{align*}
M_*^2 &= 2D_\theta ^2M_{*,\theta }^2 + 2D_y^2M_{*,y}^2 = 2R^2(M_{*,\theta }^2 + M_{*,y}^2) \\
&= R^2\left(\sigma^2 + ||A||_2^2R^2 + \big(||b||_2 + (||A||_2 + \tau)R\big)^2\right)\\
 & \le {\left( {{R^2}\left( {2||A|{|_2} + \tau } \right) + R(\sigma  + ||b|{|_2})} \right)^2},
\end{align*}
where the inequality comes from the fact that $\forall a,b,c \ge 0,a^2 + b^2 + c^2 \le (a + b +c)^2$. Thus, we may write $M_*$ as
\begin{equation}
\label{eq:Mstar-def}
M_* = {{R^2}\left( {2||A|{|_2} + \tau } \right) + R(\sigma  + ||b|{|_2})}.
\end{equation}
Now we have all the pieces ready to apply Proposition 3.2 in~\citet{RobustSA:2009} and obtain a high-probability bound on ${\rm{Err}}({\bar \theta_n},{\bar y_n})$, where $\bar \theta_n$ and $\bar y_n$ (see Eq.~\ref{eq:bartheta}) are the outputs of the revised GTD algorithm in Algorithm~\ref{alg:pgtd2}. From Proposition 3.2 in~\citet{RobustSA:2009}, if we set the step-size in Algorithm~\ref{alg:pgtd2} (our revised GTD algorithm) to $\alpha_t=\frac{2c}{M_*\sqrt{5n}}$, where $c>0$ is a positive constant, $M_*$ is defined by Eq.~\ref{eq:Mstar-def}, and $n$ is the number of training samples in $\mathcal{D}$, with probability of at least $1-\delta$, we have

\vspace{-0.15in}
\begin{small}
\begin{equation}
\label{eq:hp2}
{\rm{Err}}({\bar \theta _n},{\bar y_n}) \le \sqrt {\frac{5}{n}} (8 + 2\log \frac{2}{\delta }){R^2}\left( {2 ||A|{|_2} + \tau + \frac{{||b|{|_2} + \sigma }}{R}} \right).
\end{equation}
\end{small}
\vspace{-0.15in}

Note that we obtain Eq.~\ref{eq:hp2} by setting $c=1$ and the ``light-tail'' assumption in Eq.~\ref{eq:lt}  guarantees that we satisfy the condition in Eq.~3.16 in~\citet{RobustSA:2009}, which is necessary for the high-probability bound in their Proposition~3.2~to hold. The proof is complete by replacing $||A||_2$ and $||b||_2$ from Lemma~\ref{lem:abbound}.
\end{proof}

\section{PROOF OF PROPOSITION~\ref{pro:4} }

\begin{proof}
From Lemma~\ref{lem:v}, we have
\begin{align*}
V - {{\bar v}_n} =& \;{(I - \gamma \Pi P)^{ - 1}} \times\\
&\;\big[ {\left( {V - \Pi V} \right) + \Phi {{C}^{ - 1}}(b-A{{\bar{\theta} }_n})} \big].
\end{align*}
Applying $\ell_2$-norm w.r.t.~the distribution $\xi$ to both sides of this equation, we obtain
\begin{align}
\label{eq:ccc1}
||V - \bar v_n||_\xi \le &||(I - \gamma \Pi P)^{-1}||_\xi \times \\
&\big(||V - \Pi V||_\xi + ||\Phi C^{-1}(b - A\bar\theta_n)||_\xi\big).  \nonumber
\end{align}
Since $P$ is the kernel matrix of the target policy $\pi$ and $\Pi$ is the orthogonal projection w.r.t.~$\xi$, the stationary distribution of $\pi$, we may write 
\begin{equation*}
||{(I - \gamma \Pi P)^{ - 1}}|{|_{\xi} } \le \frac{1}{{1 - \gamma }}.
\end{equation*}
Moreover, we may upper-bound the term $||\Phi C^{-1}(b - A\bar\theta_n)||_\xi$ in~\eqref{eq:ccc1} using the following inequalities: 
\begin{align*}
||\Phi {C^{ - 1}}(b &- A{{\bar \theta }_n})|{|_\xi } \\ 
&\le ||\Phi {C^{ - 1}}(b - A{{\bar \theta }_n})|{|_2}\sqrt {{\xi _{\max }}} \\ 
&\le ||\Phi |{|_2}||{C^{ - 1}}|{|_2}||(b - A{{\bar \theta }_n})|{|_{{M^{ - 1}}}}\sqrt {\tau {\xi _{\max }}} \\
&\le (L\sqrt d )(\frac{1}{\nu })\sqrt {2{\rm{Err}}({{\bar \theta }_n},{{\bar y}_n})} \sqrt {\tau {\xi _{\max }}} \\
&= {\frac{L}{\nu }\sqrt {2d\tau {\xi _{\max }}{\rm{Err}}({{\bar \theta }_n},{{\bar y}_n})} },
\end{align*}
where the third inequality is the result of upper-bounding $||(b - A\bar \theta_n)||_M^{-1}$ using Eq.~\ref{eq:err3} and the fact that $\nu  = 1/||{C^{ - 1}}||{^2_2} = 1/{\lambda _{\max }}({C^{ - 1}}) = {\lambda _{\min }}(C)$ ($\nu$ is the smallest eigenvalue of the covariance matrix $C$).
%
\end{proof}


\section{PROOF OF PROPOSITION~\ref{pro:5} }

\begin{proof}
Using the triangle inequality, we may write 
\begin{equation}
||V-{\bar v_n}|||_\xi \le || {\bar v_n} - \Phi {\theta ^*}|{|_{\xi} } + ||V-\Phi {\theta ^*}||_\xi. 
\label{eq:errdecomp}
\end{equation} 
The second term on the right-hand side of Eq.~\ref{eq:errdecomp} can be upper-bounded by Lemma~\ref{lem:kolter}. Now we upper-bound the first term as follows: 
%
%
\begin{align*}
|| {{\bar v}_n} &- \Phi \theta^*||_{\xi} ^2\\
 &  = ||\Phi {\bar \theta _n} - \Phi {\theta ^*}||_{\xi} ^2 \\
 &  =  ||{{\bar \theta }_n} - {\theta ^*}||_C^2\\
 & \le ||{{\bar \theta }_n} - {\theta ^*}||_{{A^ \top }{M^{ - 1}}A}^2||{({A^ \top }{M^{ - 1}}A)^{ - 1}}|{|_2}||C|{|_2}\\
 & = ||A({{\bar \theta }_n} - {\theta ^*})||_{{M^{ - 1}}}^2||{({A^ \top }{M^{ - 1}}A)^{ - 1}}|{|_2}||C|{|_2}\\
 &  = ||A{{\bar \theta }_n} - b||_{{M^{ - 1}}}^2\frac{{{\tau _C}}}{{{\sigma _{\min }}({A^ \top }{M^{ - 1}}A)}},
\end{align*}
where ${\tau _C} = {\sigma _{\max }}(C)$ is the largest singular value of $C$, and ${\sigma _{\min }}({A^ \top }{M^{ - 1}}A)$ is the smallest singular value of ${{A^ \top }{M^{ - 1}}A}$. Using the result of Theorem~\ref{thm:1}, with probability at least $1-\delta$, we have
\begin{align}
\frac{1}{2}||A{\bar \theta _n} &- b||_{M^{-1}}^2 \le  \tau {\xi _{\max }}{\rm{Err}}({\bar \theta _n},{\bar y_n}).
\label{eq:thm1variant}
\end{align}
Thus,
\begin{align}
|| {\bar v_n} - \Phi \theta^*||_{\xi} ^2  
& \le \frac{{2{\tau _C}\tau {\xi _{\max }}}}{{{\sigma _{\min }}({A^ \top }{M^{ - 1}}A)}}{\rm{Err}}({\bar \theta _n},{\bar y_n})
\label{eq:barvnstar}
\end{align}
From Eqs.~\ref{eq:errdecomp},~\ref{eq:kolter2}, and~\ref{eq:barvnstar}, the result of Eq.~\ref{eq:pro5} can be derived, which completes the proof.
\end{proof}


\end{document}